\documentclass[journal]{IEEEtran}

\usepackage{graphicx}

\usepackage{algorithm}
\usepackage{algorithmic}

\usepackage{amsmath}
\usepackage{amsthm}

\usepackage{multirow,booktabs}

\DeclareMathOperator*{\argmin}{arg\,min}

\newtheorem{lemma}{Lemma}
\newtheorem{theorem}{Theorem}
\newtheorem{corollary}{Corollary}

\begin{document}

\title{Minimax Classifier for Uncertain Costs}
\author{
    Rui~Wang, Ke~Tang
    \thanks{Nature Inspired Computation and Applications Laboratory (NICAL),
            Hefei, Anhui 230027 China,
            E-mail: wrui1108@mail.ustc.edu.cn, ketang@ustc.edu.cn.
            }
}
\maketitle

\begin{abstract}
Many studies on the cost-sensitive learning assumed that a unique cost matrix is known for a problem. However, this assumption may not hold for many real-world problems. For example, a classifier might need to be applied in several circumstances, each of which associates with a different cost matrix. Or, different human experts have different opinions about the costs for a given problem. Motivated by these facts, this study aims to seek the minimax classifier over multiple cost matrices. In summary, we theoretically proved that, no matter how many cost matrices are involved, the minimax problem can be tackled by solving a number of standard cost-sensitive problems and sub-problems that involve only two cost matrices. As a result, a general framework for achieving minimax classifier over multiple cost matrices is suggested and justified by preliminary empirical studies.
\end{abstract}

\section{Introduction}
In many real world classification problems, different types of misclassifications commonly result in different costs. For example, in fraud detection problem, predicting a normal client as fraud will cut the profit, while predicting a fraud client as normal would usually lead to great loss \cite{Bolton_Hand_2002}. In these scenarios, it would be more desirable to minimize the total cost rather than the classification error. This kind of problem is referred to as cost sensitive-learning problem  \cite{Elkan2001}, and has attracted many interests in recent years due to its wide applications in the real world \cite{AbeZL04, Zhou2006, Dmochowski2010}.

So far, the majority of previous research on cost-sensitive learning assumes that the costs for different types of misclassifications, typically represented as a cost matrix, are uniquely specified before the classifier is applied to new data. Specifically, if the cost matrix is known before the training procedure, it can be integrated into the the learning algorithm to obtain a classifier with minimum total cost. This can be done by modifying the training data according to the cost matrix \cite{Domingos1999, Ting2002}, or by extending learning algorithms directly \cite{ZhouLiu_TKDE_2006, Masnadi-shirazi_Vasconcelos_Member_2011}. In addition to specialized methods, some alternative approaches, which are motivated by other learning problems, could also be employed to address cost-sensitive learning problems. This category of methods, including calibration methods \cite{Zadrozny_Elkan_2002}, threshold moving \cite{Dmochowski2010} and its variants \cite{BourkeKunS2008}, typically post process the output of a classifier to optimize its performance with respect to a objective (e.g., minimize the total cost or classification error). In this sense, it is not necessary to know the cost matrix in prior to the training phase, as long as it becomes available before testing \footnote{Sometimes, post-processing can also be considered as a part of the training procedure. From this point of view, the cost matrix still needs to be specified before training phase is finished.}.

All the above-mentioned approaches assume that a unique cost matrix is known for a given cost-sensitive problem. Unfortunately, in the real world, it could be very difficult for a practitioner to specify the cost matrix uniquely, for the reason that one may do not have much sense about the exact values of misclassification costs, or that the costs may vary under different circumstances and thus is uncertain in nature. In one word, the cost matrix for a real-world problem may be uncertain throughout both training and testing.

As a matter of fact, the difficulty of specifying a cost matrix has been acknowledged by many researchers. In the context of ROC analysis \cite{Provost2001}, it is claimed that a classifier can be built without any cost information, while still performs well in the scenarios where the cost matrix changes. Nevertheless, an underlying assumption behind this statement is that threshold moving (or any other similar methods) is employed to fine-tune the output of the classifier. Hence, as discussed above, the specified cost matrix is still required in the post-processing phase. Zadrozny and Elkan \cite{Zadrozny_Elkan_KDD_2001} considered the scenario where example-based misclassification costs are static but unknown. More recently, Liu and Zhou \cite{LiuZhou2010} investigates the problem of learning with cost intervals. Specifically, the misclassification cost is assumed as taking a value within a predefined interval, and an approach is developed to train a SVM that performs well for every possible value of cost.

Rather than striving to achieve satisfactory performance over all possible cost matrices, the aim of this work is to minimize the largest total cost over a finite set of possible cost matrices, i.e., to find the minimax classifier. Under mild assumptions, we prove that the minimax classifier over multiple cost matrices can be achieved by solving a set of standard cost-sensitive learning problems and a set of sub-problems involves only two cost matrices. This finding immediately suggests a general framework for seeking minimax classifier over arbitrary number of cost matrices. Moreover, since an interval can be transformed into a finite set of values via discretization, the framework is also applicable to the scenarios where only the largest and smallest costs for misclassification are available.

The rest of this paper is organized as follows. Preliminary backgrounds and related works are introduced with more details in Section 2. Section 3 presents the theoretical analysis of the minimax problem. Experimental studies are in following section, and we conclude the paper in Section 5.

\section{Preliminaries and Related Works}
In this section, we introduce the basic notations and backgrounds at first, and then review two works that are closely related to this study. One is the work from Liu and Zhou \cite{LiuZhou2010}, which also deals with the uncertain cost problem, but with different formulation and learning target, and the other focus on finding the minimax classifier for uncertain class prior \cite{Rodriguez2005}.

\subsection{Preliminaries}
Given a dataset $S=\{(x_1, y_1), \ldots, (x_n, y_n)\}$, $x_i=\{x_i^1, \ldots, x_i^m\}\in{\mathcal{R}^m}$ is the feature vector of instance $(x_i, y_i)$ and $y_i\in\{0, 1\}$ is the class label. Suppose that there is no cost with correct classification, a cost matrix $C$ can be represented by two values $c_0$ and $c_1$, denoting the cost of misclassifying an instance from class $0$ and class $1$ respectively. Also, we use $p_0$ and $p_1=1-p_0$ to represent the class priors, so there are $n_0=np_1$ and $n_1=np_1$ instances in each class. For any classifier $h$ from the hypothesis space $\mathcal{H}$, its total cost is,
\begin{equation}
\label{tc}
\begin{aligned}
  L&=np_0p_{10}c_0+np_1p_{01}c_1\\
  &=n_0p_{10}c_0+n_1p_{01}c_1,
\end{aligned}
\end{equation}
where $p_{10}\,(p_{01})$ is the probability that $h$ misclassifies instances from class $0\,(1)$ to class $1\,(0)$.

\subsection{Learning with Cost Intervals}
In a recent work, Liu and Zhou \cite{LiuZhou2010} considered a special form of the uncertain cost problem where $c_0$ is $1$, and $c_1$ is uncertain but within a predefined interval $[c_{min}, c_{max}]$. Their objective is to construct a classifier that performs well for every individual cost within $[c_{min}, c_{max}]$. Technically, the problem was transformed as finding the best surrogate cost $c_s$ to trained with, i.e., their learning target is,
\begin{equation}
 \begin{aligned}
  &\min_{h\in{\mathcal{H}}}{L(h, S, c_s)} \\
  &s.t.\,\,\,p(L(h,S,c) < \epsilon)>1-\delta, \forall\,c\in [c_{min}, c_{max}]\\
  &\,\,\,\,\,\,\,\,\,\,\,c_{min}\leq{c_s}\leq{c_{max}}
 \end{aligned}
\end{equation}
A SVM-based algorithm was proposed there, which primarily minimizes the largest total cost (i.e., $L(h, S, c_{max})$) and secondarily minimizes the total cost at mean cost $c_\mu=(c_{max}+c_{min})/{2}$, i.e., $L(h, S, c_\mu)$. Solid experimental results reported there confirmed the efficacy of the method.

However, to fit in with the interval formulation, one needs to artificially re-scale original cost matrices by different factors to assure every $c_0$ is $1$. Although this re-scaling process does no harm to traditional cost-sensitive learning as well as the study in \cite{LiuZhou2010}, it makes the comparison of total costs across different cost matrices meaningless. Considering that it is generally hard or even impossible to find a classifier that performances well on all costs over the interval (as suggested by \cite{LiuZhou2010} itself), the \emph{best} classifier they built may lead to very big total cost on original cost matrices for real-world problems.

\subsection{Minimax Classifier for Uncertain Class Priors}
In the many studies involving the minimax criterion \cite{DuPardalos1995}, those focused their attention on building minimax total cost classifier for uncertain class prior \cite{Rodriguez2005} are of particular interest to this study.

Formally, in case of uncertain class prior, the minimax classification problem is to find the following classifier,
\begin{equation}
\label{minimaxPrior}
h_p=\argmin_{h\in{\mathcal{H}}}\max_{P}{L(h,P,C)}
\end{equation}
It is well known that the total cost of a fixed classifier is a linear function of prior, while the optimal total cost (i.e., the Bayesian cost) is a concave function of prior \cite{Duda2001}. Therefore, suppose the best classifier is $h^*$ for a given class prior $P^*$, then the total cost function of $h^*$ w.r.t. prior would be a tangent line of the Bayesian total cost curve at $P^*$. Based on these elegant properties, Alaiz-Rodriguez et al proposed two algorithms based on neural networks model to find the minimax classifier iteratively in \cite{Rodriguez2005}. Readers interested in the details of the algorithms are referred to that paper.

Notice the deceptively symmetrical positions of prior and cost in Eq. (\ref{tc}), one may think that all the analysis and algorithms w.r.t. the uncertain prior problem can be employed directly for the uncertain cost problem concerned in this study. Unfortunately, that is not the case. For the reason that both $c_0$ and $c_1$ are free variables (i.e., the sum-to-one property of prior does not applied to cost), the concavity of Bayesian total cost for prior can not be transformed to cost. In the following, we consider the minimax problem for uncertain cost along a different way.

\section{Minimax Classifier for Uncertain Cost}
\subsection{Problem Formulation}
As mentioned above, this study focuses on minimizing the largest total cost over a finite set of possible cost matrices. Formally, given a set of cost matrices $U=\{C_1, \ldots C_k\}$, where $C_i=\{c^i_0, c^i_1\}$ is the $i$-th cost matrix, the learning target is to find,
\begin{equation}
    \label{minimaxTc}
    h_U=\argmin_{h\in{\mathcal{H}}}{\max_{C\in{U}}{L(h, S, C)}}.
\end{equation}
Since the uncertain cost is formulate as a set directly, the problem is widely applicable in practice, ready for future study on multi-class problems, and facilitating theoretical analysis. On the other hand, the best classifier selected by the minimax criterion is much more reliable.
\subsection{Problem Analysis}
\label{analysis}
For two different cost matrices $C_i$ and $C_j$ in $U$, if both $c_0^i\leq{c_0^j}$, and $c_1^i\leq{c_1^j}$, then the total cost of any classifier $h$ obtained on $C_i$ will be smaller than that on $C_j$. In this case, we say that $C_i$ is dominated by $C_j$. Furthermore, if there exist a cost matrix $C_d$ that dominates all others in $U$, the above minimax problem can be simplified as a standard cost-sensitive learning problem with fixed cost matrix $C_d$. Therefore,  given a minimax classification problem over a set of cost matrices, the first step one should take is to check and delete cost matrices that are dominated by any other cost matrix in $U$.

On the other hand, the performance of a classifier $h$ from the hypothesis space $\mathcal{H}$ can be mapped to a point in the $2$-D space with $p_{10}$ as the \emph{x}-axis and $p_{01}$ as the \emph{y}-axis. Similarly, for two different classifiers $h_a$ and $h_b$, if $p^{h_a}_{10}\leq{p^{h_b}_{10}}$, and $p^{h_a}_{01}\leq{p^{h_b}_{01}}$, then $L_{h_a}\leq{L_{h_b}}$, no matter what the cost matrix is. In this case, we say that $h_a$ dominates $h_b$. If a classifier is not dominated by any other classifier in $\mathcal{H}$, it is a non-dominated classifier. Following the concept in economics \cite{Pareto1964}, the front formed by all non-dominated classifiers in $\mathcal{H}$ is named as the \emph{Pareto front} (see Fig. \ref{pf}). When $\mathcal{H}$ is an infinite hypothesis space and dataset $S$ consist of enough samples, the front is continues. Obviously, for both standard cost-sensitive learning problem and the minimax problem concerned in this study, the optimal classifiers must be on the Pareto front.
\begin{figure}[ht]
        \begin{center}
        \centerline{\includegraphics[width=0.7\columnwidth]{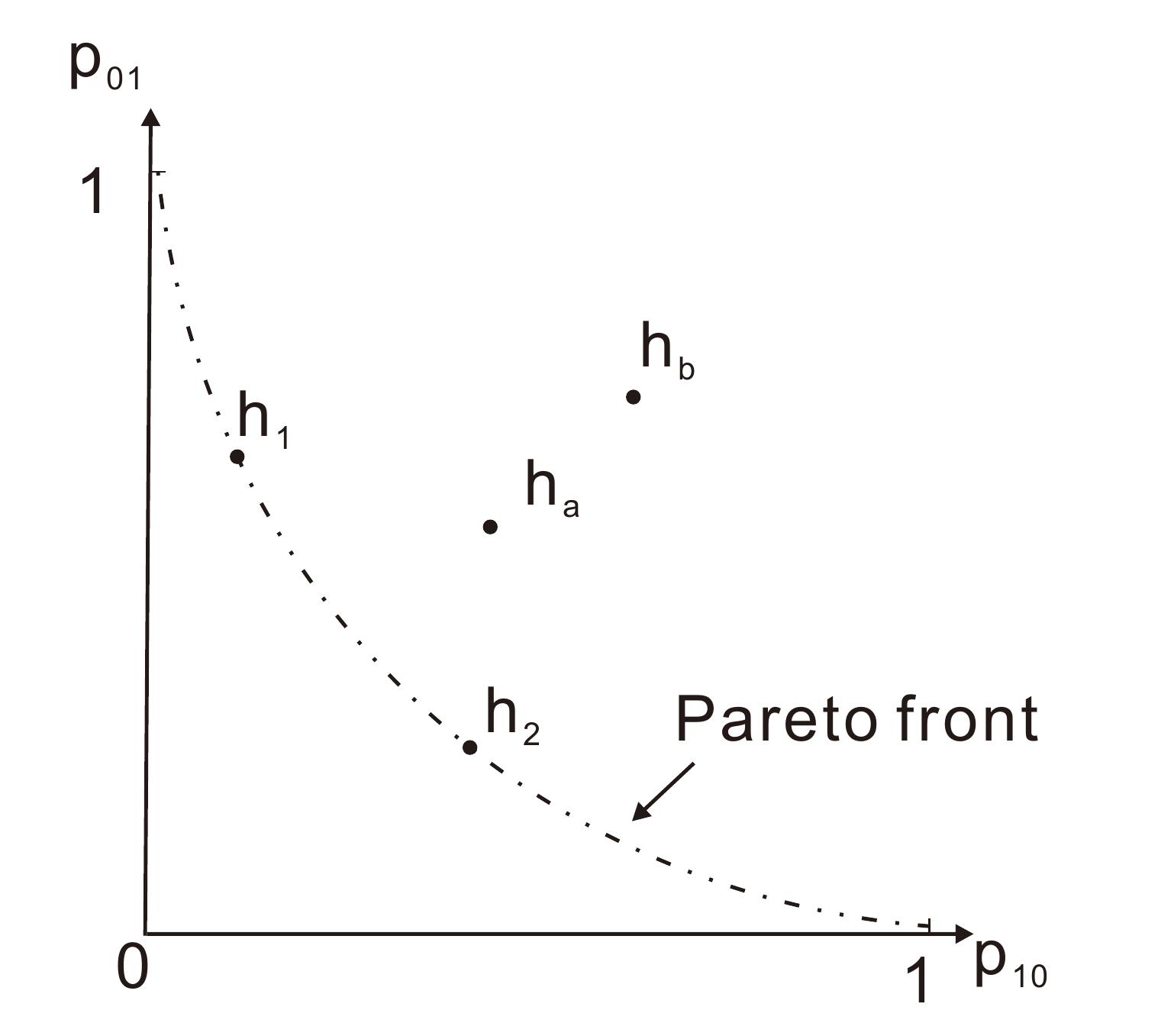}}
        \caption{Mapping classifiers to the $2$-D space with $p_{10}$ as the \emph{x}-axis and $p_{01}$ as the \emph{y}-axis. $h_a$ dominates $h_b$. $h_1$, $h_2$ are non-dominated classifiers, hence on the Pareto front.}
        \label{pf}
        \end{center}
        \vskip -0.15in
\end{figure}\\
Let us firstly consider the situation with only one cost matrix $C$, Lemma \ref{lemma1} reveals the relative order between the total costs of any two classifiers on the front.
\begin{lemma}
    \label{lemma1}
    For any two classifiers $h_1$, $h_2$ on the Pareto front, with $p^{h_1}_{10} \leq{ p^{h_2}_{10}}$ and $p^{h_1}_{01}\geq{  p^{h_2}_{01}}$, the following conclusions hold,
    \begin{itemize}
      \item
      $
        \frac{p^{h_1}_{01}-p^{h_2}_{01}}{p^{h_2}_{10}-p^{h_1}_{10}}>\frac{n_0c_0}{n_1c_1} \iff L_{h_1}>L_{h_2},
      $
      \item
      $
        \frac{p^{h_1}_{01}-p^{h_2}_{01}}{p^{h_2}_{10}-p^{h_1}_{10}}=\frac{n_0c_0}{n_1c_1} \iff L_{h_1}=L_{h_2},
      $
      \item
      $
        \frac{p^{h_1}_{01}-p^{h_2}_{01}}{p^{h_2}_{10}-p^{h_1}_{10}}<\frac{n_0c_0}{n_1c_1} \iff L_{h_1}<L_{h_2}.
      $
    \end{itemize}
\end{lemma}
\begin{proof}
    According to Eq. (\ref{tc})
    \[
    L_{h_1} = n_0c_0p^{h_1}_{10}+n_1c_1p^{h_1}_{01}
    \]
    \[
    L_{h_2} = n_0c_0p^{h_2}_{10}+n_1c_1p^{h_2}_{01}
    \]
    Therefore,
    \[
    \begin{array}{rr}
    L_{h_1}>L_{h_2} \iff  &\\
    n_0c_0p^{h_1}_{10}+n_1c_1p^{h_1}_{01} > n_0c_0p^{h_2}_{10}+n_1c_1p^{h_2}_{01} \iff  &\\ \frac{p^{h_1}_{01}-p^{h_2}_{01}}{p^{h_2}_{10}-p^{h_1}_{10}}>\frac{n_0c_0}{n_1c_1}&
    \end{array}
    \]
    Similarly, the other two cases also validated.
  \end{proof}
Notice that the left hand of each case in Lemma \ref{lemma1}, $(p^{h_1}_{01}-p^{h_2}_{01})/(p^{h_2}_{10}-p^{h_1}_{10})$, is the abstract value of the slope of the segment connected $h_1$ and $h_2$, which is determined by classifiers' performance, and $(n_0c_0)/(n_1c_1)$, on the other hand, is a constant given dataset $S$ and cost matrix $C$. That is, geometrically, the relative order between the total costs of  a pair of classifiers on the front is determined by slope of the segment connected these two classifiers.

Furthermore, since all the dominated classifiers can be ignored w.r.t. our problem, total cost Eq.(\ref{tc}) can be treated as a function of the classifiers on the Pareto front. For briefness, we further consider it as a function of $p_{10}$, and keep in mind that $p_{01}$ is determined correspondingly. Hereafter, we denote the total cost function as $L_C(p_{10})$ for cost matrix $C$. The following lemma describes the track of $L_C(p_{10})$ along the front.
  \begin{lemma}
  \label{lemma2}
    Assume the Pareto front is convex\footnote{Analogous to the \emph{ROCCH} technique, in case that the Pareto front is not convex, one can construct the convex hull of all non-dominated classifiers as the surrogate Pareto front. Please refer to \cite{Provost2001}, particularly Theorem 7 there, for further details.}, then the total cost function $L_C(p_{10})$ decreases monotonically to its minimum at first, and then increases monotonically over the front.
  \end{lemma}
  \begin{proof}
    Given any three adjacent classifiers on the front, $h_1, h_2, h_3$, without loss of generality, we suppose $p^{h_1}_{10}<p^{h_2}_{10}<p^{h_3}_{10}$. Since the curve of the Pareto front is decreasing and convex, $h_3$ must lay on the right-side of the line passes $h_1$ and $h_2$. Plus the fact that $p^{h_1}_{10}<p^{h_2}_{10}<p^{h_3}_{10}$ and $p^{h_1}_{01}>p^{h_2}_{01}>p^{h_3}_{01}$, we have
    \[
    \frac{p^{h_1}_{01}-p^{h_2}_{01}}{p^{h_2}_{10}-p^{h_1}_{10}}\ge\frac{p^{h_2}_{01}-p^{h_3}_{01}}{p^{h_3}_{10}-p^{h_2}_{10}}.
    \]
    Since $h_1, h_2, h_3$ are three arbitrary adjacent classifiers on the front, it comes that the abstract value of the slope is adjacently and monotonically non-increasing along the front.\\
    Suppose the classifier of minimal total cost for cost matrix $C$ is $h_C$, with the total cost $L_C(p_{10}=p^{h_C}_{10})$, then according to Lemma \ref{lemma1}, $L_C(p_{10})$ decreases monotonically to $(p^{h_C}_{10}, L_C(p_{10}))$ at first, and then increases monotonically.
  \end{proof}
In fact, Lemma \ref{lemma2} describes the behavior of total cost function for standard cost-sensitive learning problem (i.e., with only one cost matrix), and many cost-sensitive learning methods published in the literature could be used, hopefully, to find the minimum point. See Fig. \ref{tcvsp10} for an illustration\footnote{Note the total cost curve was drawn for illustration purpose, the convexity it appears is not implied nor has been proved.}.
\begin{figure}[ht]
        \begin{center}
        \centerline{\includegraphics[width=0.7\columnwidth]{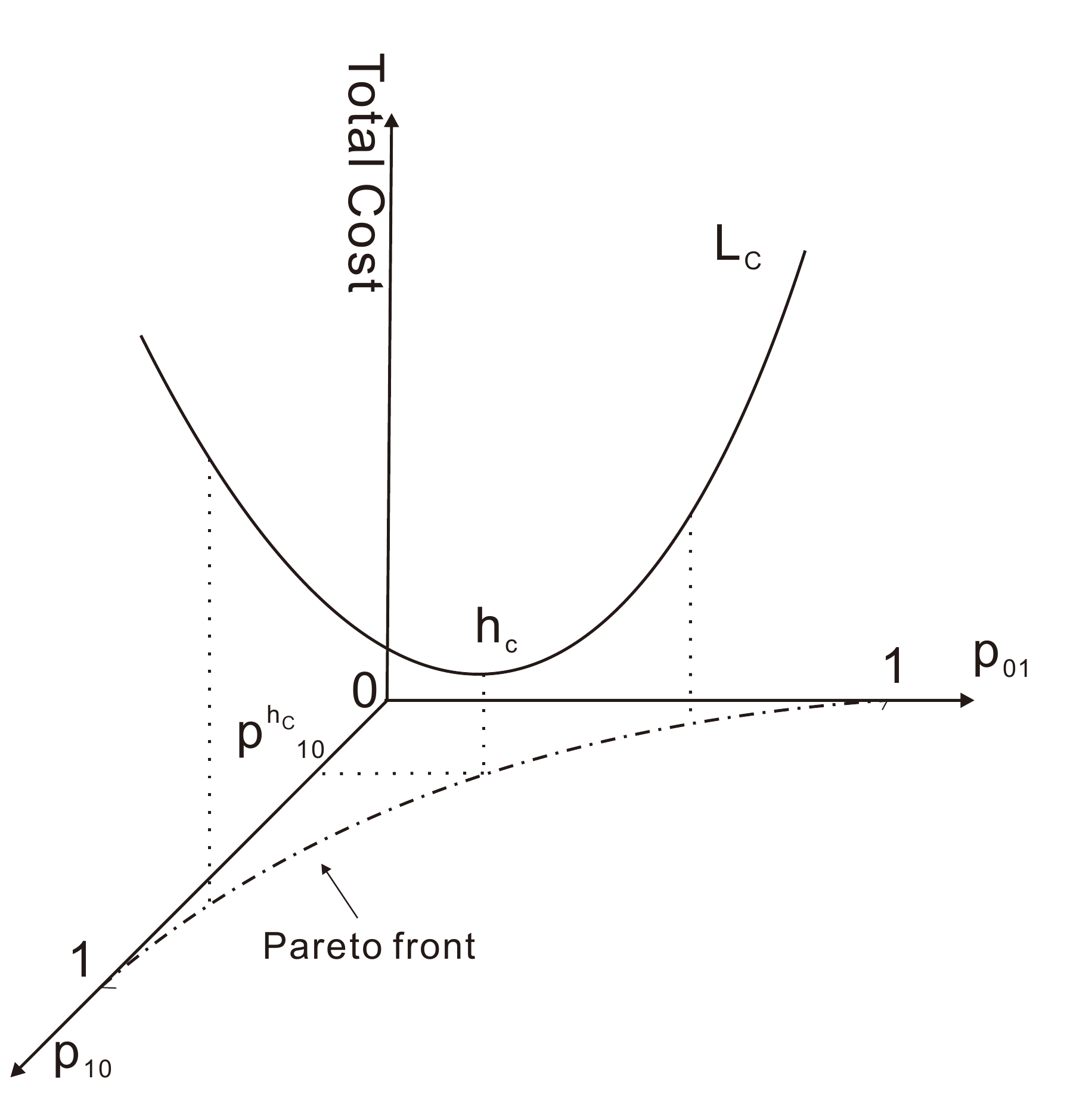}}
        \caption{The total cost curve vs. classifiers on the Pareto front. From the perspective of $p_{10}$, it also decreases at first, and then increases.}
        \label{tcvsp10}
        \end{center}
        \vskip -0.15in
\end{figure}
\begin{figure}[ht]
        \vskip 0.05in
        \begin{center}
        \centerline{\includegraphics[width=0.7\columnwidth]{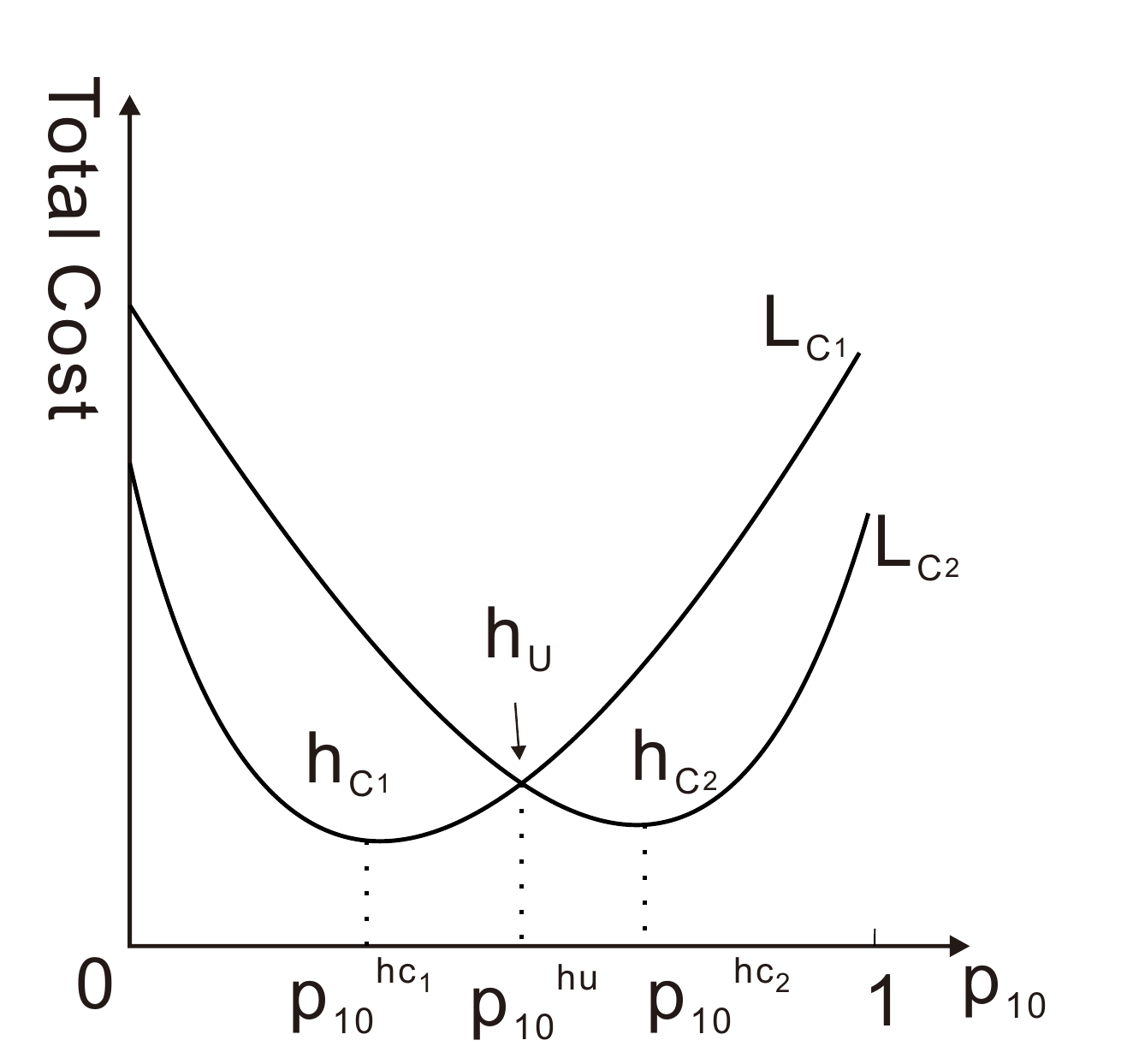}}
        \caption{The total cost curves for $C_1$ and $C_2$. Each of them decreases at first, and then increases monotonously. $h_1$ and $h_2$ are the best classifiers for $C_1$ and $C_2$, and $h_c$ is the minimax classifier for these two cost matrices.}
        \label{minimax2c}
        \vskip -0.15in
        \end{center}
\end{figure}\\
Now, we are ready for considering the situation with multiple cost matrices. For a set of $k$ cost matrices, there are $k$ total cot curves correspondingly. Each of them decreases to its own minimum at first and then increases. Fig. \ref{minimax2c} shows a example consist of two cost matrices. We can see that, in this case, the minimax total cost locates at the cross point of these two curves. Generally, the position of the minimax classifier for multiple cost matrices is confined by the following theorem.
\begin{theorem}
  \label{theorem}
  For $k$ different total cost curves, each of them decreases to its own minimum at first and then increases monotonically, the minimax total cost locates at one of the two types of positions,
    \begin{enumerate}
      \item minimum point of an individual curve,
      \item point where curves get crossed.
    \end{enumerate}
\end{theorem}
\begin{proof}
    Suppose the minimax total cost locates at neither one of the two types of positions, without loss of generality, we assume it is on the total cost curve of $C_i$ (i.e., $L_{C_i}$). Note that the minimax classifier is $h_U$, we have,
    \[
    L_{C_i}(p^{h_U}_{10})>L_{C_j}(p^{h_U}_{10}) \, \text{for}\, i\neq{j}.
    \]
    There is no equality because the the minimax total cost is not obtained at a type-2 point. On the other hand, since the minimax total cost is obtained neither at a type-1 point, according to Lemma \ref{lemma2}, we can find another classifier $h'_U$ such that,
    \[
    L_{C_i}(p^{h_U}_{10})>L_{C_i}(p^{h'_U}_{10})>L_{C_j}(p^{h'_U}_{10}) \, \text{for}\, i\neq{j}.
    \]
    This means the minimax total cost can be reduced, which conflicts with the definition of minimax. Therefore, the theorem is validated.
\end{proof}
 So, in order to find the minimax classifier, we just need to examine every classifier corresponds to the two types of positions. However, without further information, any pair of total cost curves may cross each other several times in practice, hence it would be very expensive or even impossible to examine all these points without omission. Fortunately, this obstacle can be removed elegantly by the following corollary.
\begin{corollary}
  \label{corollary}
    For a set of $k$ cost matrices $U=\{C_1, \ldots, C_k\}$, the minimax total cost classifier $h_U$ belongs to one of following two categories,
    \begin{enumerate}
      \item classifiers that minimize the total cost for an individual cost matrix,
      \item classifiers that \textbf{minimax} the total cost for a pair of cost matrices.
    \end{enumerate}
\end{corollary}
\begin{proof}
  According to Theorem \ref{theorem}, if the minimax total cost is obtained at one of the type-1 positions, then the minimax classifier fall into the first category, thus the corollary is true. \\
  Otherwise, the minimax total cost is obtained at a cross point of total cost curves. We know that there are at least two total cost curves with different monotonic property at the cross point, otherwise, we can move $h_U$ in the direction that all involved curves are decreasing, leading to reduced minimax total cost. Let the $L_{C_i}$ is decreasing, and $L_{C_j}$ is increasing at $h_U$, then we know from Lemma \ref{lemma2} that the maximal total cost for $(C_i, C_j)$ is bigger with all other classifiers. Hence, the cross point is the also the minimax total cost for $(C_i, C_j)$. So, the corollary is also true in this case.
\end{proof}

According to Corollary \ref{corollary}, the minimax classification problem over multiple cost matrices is reduced to solving a set of standard cost-sensitive learning problems and a set of sub-problems involves only two cost matrices, saving the bother to consider the tradeoff among multiple cost matrices. Finally, the framework for solving the minimax classification problem over a set of cost matrices is summarized in Algorithm \ref{algoFramework}.

\begin{algorithm}[h]
   \caption{Framework of solving the minimax classification problem over a set of cost matrices}
   \label{algoFramework}
\begin{algorithmic}
   \STATE {\bfseries Input:} dataset $S$, a set of cost matrices $U=\{C_1, \ldots, C_k\}$
   \STATE deletes all dominated cost matrices in $U$,
   \IF{$U=\{C_1\}$}
   \STATE $h_U=\argmin_{h\in\mathcal{H}}L(h, S, C_1)$
   \ELSIF{$U=\{C_1, C_2\}$}
   \STATE $h_U=\argmin_{h\in\mathcal{H}}\max_{C\in\{C_1, C_2\}}L(h, S, C)$
   \ELSE
   \STATE $V=\emptyset$
   \FOR{$i=1$ {\bfseries to} $|U|$}
   \STATE find $h_{C_i}=\argmin_{h\in\mathcal{H}}{L(h, S, C_i)}$
   \STATE $V=V\bigcup\{h_{C_i}\}$
   \ENDFOR
   \FOR{$i=1$ {\bfseries to} $|U|$}
   \FOR{$j=1+1$ {\bfseries to} $|U|$}
   \STATE find $h_{ij}=\argmin_{h\in\mathcal{H}}\max_{C\in\{C_i, C_j\}}{L(h, S, C)}$
   \STATE $V=V\bigcup\{h_{ij}\}$
   \ENDFOR
   \ENDFOR
   \STATE $h_U=\argmin_{h\in{V}}\max_{C\in{U}}{L(h, S, C)}$
   \ENDIF
   \STATE{\bfseries Return:} the minimax classifier $h_U$
\end{algorithmic}
\end{algorithm}

\section{Experiments}
In the experiments, we compared three frameworks for solving the minimax problem. The first is to build the minimum total cost classifier for each possible cost matrix without considering any tradeoff among cost matrices at first, and then picks out the minimax classifier, the second is our framework described above, and the third one is to build the minimax classifier directly with all the possible cost matrices are under consideration simultaneously. For briefness, we denote these three frameworks as \textbf{S}, \textbf{SP}, and \textbf{M} respectively.

\subsection{Implementation}
Although there are many standard cost-sensitive learning methods, striving to minimize the total cost for one cost matrix, can be used to implement \textbf{S} and one part of \textbf{SP}, to the best of our knowledge, there is no particular method that can be used to implement \textbf{M} or the the other part of \textbf{SP} (i.e., minimax the total cost for two or more cost matrices). Hence, for the comparison purpose, we adopted a simplified form of the Generalized Additive Model (GAM) to implement all the three frameworks. Therefore, the empirical studies presented underneath are preliminary, and only intend to serve as a baseline for future study.

The GAM used to implement all the compared frameworks is,
\begin{equation}
  F(x)=sign(\sum_{i=1}^{T}f_i(x))
\end{equation}
where $T$ is the number of iterations, and $f_i$ is a decision stump, whose output is $1$ or $-1$. At each generation, we add one decision stump such that the current ensemble of decision stump $F_i$ get improved performance over $F_{i-1}$ on the predefined objective. This process repeats until the iteration number is ran out or there is no improvement.

With this simple GAM procedure, we are able to implement the three above-mentioned frameworks. That is, all necessary building blocks can be generated by setting the ``predefined objective'' to minimize the total cost for a single cost matrix, or minimax the total cost for a pair of cost matrices, or minimax the total cost for a set of cost matrices.

\subsection{Experimental Setup}
Ten datasets from the UCI machine learning repository \cite{Asuncion2007} were used in the experiments. Brief information about these datasets is summarized in Table \ref{datasets}.
\begin{table}
  \caption{Summary of the 10 datasets used in the empirical studies}
      \centering
      \begin{tabular}{lcc}
        \toprule
        Dataset& No. of Features & Class Distribution\\
        \midrule
        australian & 14 & 307:383\\
        crx  & 16 & 307:383\\
        german & 24 & 700:300\\
        heart& 13 & 150:120\\
        hill-valley & 100 & 612:600\\
        house-votes & 16 & 168:267\\
        kr-vs-kp & 36 & 1669:1527\\
        mushroom & 22 & 3916:4208\\
        sonar & 60 & 97:111\\
        wdbc & 30 & 357:212\\
        \bottomrule
      \end{tabular}
      \label{datasets}
\end{table}
Most of these ten classification problems are originally real-world cost-sensitive problems, for example the \emph{australian}, \emph{crx}, and \emph{german} problems are fraud detection problems, while the \emph{heart}, \emph{mushroom} and \emph{wdbc} problems are related to health of people. For these problems, the misclassification cost matrix is usually hard if not impossible to specified by practitioners, so the experiments on them are appropriate.

For each of the datasets, we compared the 3 frameworks on 4 set of cost matrices of different cardinalities. They are sets of 3 cost matrices, 5 cost matrices, 10 cost matrices, and 20 cost matrices. The value of each element of the matrices is randomly generated within $[0, 10)$. Besides, it is assured in advance that there is no dominated cost matrix in each set.

The iteration number in the GAM is set to $50$, and 20 times 5-fold cross validation procedure was employed to obtain stationary results. Hence, for each of $10\times{4}\times{3}$ configurations of (dataset, set of cost matrices, compared method), there are $20\times{5}$ total cost values. Based on these values, we furthermore conducted Wilcoxon signed rank test between \textbf{SP} method and the other two methods with significance level $5\%$.

\subsection{Results}
Table \ref{trainMMtc} and Table \ref{testMMtc} present the comparisons over each dataset on training and testing respectively. The value in each cell is the average total cost over 20 times 5-fold, and the best performance for each (dataset, cost set) configuration is in boldface. Moreover, the results of Wilcoxon signed rank test are denoted as superscripts on the values of \textbf{S} and \textbf{M} methods, a superscript of $1$ indicates the performance of \textbf{SP} is significantly better than that of corresponding method, $-1$ for significantly worse, and no superscript means there is no statistically significant difference between \textbf{SP} and corresponding method.

In summary, we can see that \textbf{SP} outperforms the other two methods in almost all cases, and keeps statistically comparable for the rest few cases. There is no case that \textbf{SP} is statistically worse (i.e., there is no $-1$ on the superscripts).

Of course, it is not surprising at all that \textbf{SP} defeats \textbf{S} completely in the experiments, since \textbf{SP} always checks a superset of classifiers compared to \textbf{S}. But these results at least provide a evidence that the \textbf{S} framework is not adequate for uncertain costs problems. Moreover, with a closer examination of the results in each fold, we can see that the performance of \textbf{S} and \textbf{SP} are identical sometimes, and \textbf{SP} is better if they are not. This is consistent with Corollary \ref{corollary}, since the classifier obtained by \textbf{S} could be the optimal minimax classifier in theory.

On the other hand, the superior of \textbf{SP} over \textbf{M} is more interesting. Unlike the \textbf{S} framework, \textbf{M} searches the hypothesis space with the true learning target directly (i.e., the minimax target). Therefore, the most plausible explanation is that the implementation of the \textbf{M} framework is no effective enough. Since ideally it could perform as good as the \textbf{SP} framework. However, as the similar problem encountered in multi-class classification problems, designing algorithms that can handle multiple tradeoff simultaneously is never a trivial work.
\begin{table*}[ht]
\caption{Total costs of the three methods on each dataset with different number of cost matrices during training}
\newsavebox{\rstTrain}
\begin{lrbox}{\rstTrain}
\begin{tabular}{l|*{3}{c}|*{3}{c}|*{3}{c}|*{3}{c}}
\toprule
 & \multicolumn{3}{c|}{3 cost matrices} & \multicolumn{3}{c|}{5 cost matrices} & \multicolumn{3}{c|}{10 cost matrices} & \multicolumn{3}{c}{20 cost matrices}\\
\midrule
DataSet & S & SP & M & S & SP & M & S & SP & M & S & SP & M\\
\midrule
australian & $1440.49^1$ & $\textbf{1367.74}$ & $1466.33^1$ & $1869.28^1$ & $\textbf{1587.75}$ & $2199.41^1$ & $1916.21^1$ & $\textbf{1515.35}$ & $2363.94^1$ & $2096.11^1$ & $\textbf{1439.85}$ & $2481.22^1$\\
crx & $282.02^1$ & $\textbf{272.43}$ & $937.48^1$ & $1758.45^1$ & $\textbf{1533.96}$ & $1962.19^1$ & $1898.56^1$ & $\textbf{1420.8}$ & $2187.72^1$ & $2170.47^1$ & $\textbf{1463.45}$ & $2432.3^1$\\
german & $2625.29^1$ & $\textbf{2249.45}$ & $3060.18^1$ & $2099.12^1$ & $\textbf{1998.53}$ & $3727.99^1$ & $1895.13^1$ & $\textbf{1593.92}$ & $3733.35^1$ & $1995.7^1$ & $\textbf{1752.62}$ & $4060.1^1$\\
heart & $435.88^1$ & $\textbf{344.39}$ & $605.44^1$ & $413.54^1$ & $\textbf{345.72}$ & $619.3^1$ & $302.15^1$ & $\textbf{214.9}$ & $667.01^1$ & $333.18^1$ & $\textbf{278.55}$ & $768.89^1$\\
hill-valley & $3109.49^1$ & $\textbf{2227.36}$ & $2228.87^1$ & $2370.07^1$ & $\textbf{2070.78}$ & $\textbf{2070.78}$ & $3095.08^1$ & $\textbf{2634.57}$ & $2634.97$ & $3809.36^1$ & $\textbf{2870.9}$ & $\textbf{2870.9}$\\
house-votes & $670.16^1$ & $\textbf{458.55}$ & $823.89^1$ & $807.34^1$ & $\textbf{687.18}$ & $1001.34^1$ & $772.51^1$ & $\textbf{449.74}$ & $1004.26^1$ & $575.23^1$ & $\textbf{387.44}$ & $1087.35^1$\\
kr-vs-kp & $9539.25^1$ & $\textbf{8309.86}$ & $9976.34^1$ & $7025.3^1$ & $\textbf{6495.6}$ & $9237.13^1$ & $6409.04^1$ & $\textbf{4561.48}$ & $9372.19^1$ & $7649.38^1$ & $\textbf{5423.56}$ & $11482.27^1$\\
mushroom & $22280.14^1$ & $\textbf{14009.54}$ & $26392.35^1$ & $14268.4^1$ & $\textbf{10101.95}$ & $17715.29^1$ & $20820.21^1$ & $\textbf{10103.79}$ & $23435.58^1$ & $16941.87^1$ & $\textbf{8202.44}$ & $24488.3^1$\\
sonar & $551.62^1$ & $\textbf{489.32}$ & $571.2^1$ & $280.42^1$ & $\textbf{257.66}$ & $382.92^1$ & $451.66^1$ & $\textbf{369.81}$ & $600.6^1$ & $316.12^1$ & $\textbf{236.27}$ & $404.28^1$\\
wdbc & $585.49$ & $\textbf{549.76}$ & $1429.25^1$ & $450.72^1$ & $\textbf{375.75}$ & $1277.75^1$ & $573.19^1$ & $\textbf{331.63}$ & $1501.69^1$ & $404.9^1$ & $\textbf{233.12}$ & $1439.93^1$\\
\bottomrule
\end{tabular}
\end{lrbox}
\scalebox{0.7}{\usebox{\rstTrain}}
\label{trainMMtc}
\end{table*}

\begin{table*}[ht]
\caption{Total costs of the three methods on each dataset with different number of cost matrices during testing}
\newsavebox{\rst}
\begin{lrbox}{\rst}
\begin{tabular}{l|*{3}{c}|*{3}{c}|*{3}{c}|*{3}{c}}
\toprule
 & \multicolumn{3}{c|}{3 cost matrices} & \multicolumn{3}{c|}{5 cost matrices} & \multicolumn{3}{c|}{10 cost matrices} & \multicolumn{3}{c}{20 cost matrices}\\
\midrule
DataSet & S & SP & M & S & SP & M & S & SP & M & S & SP & M\\
\midrule
australian & $363.1^1$ & $\textbf{350.57}$ & $374.96^1$ & $471.13^1$ & $\textbf{403.58}$ & $550.78^1$ & $483.6^1$ & $\textbf{390.21}$ & $595.37^1$ & $526.1^1$ & $\textbf{373.01}$ & $623^1$\\
crx & $71.19$ & $\textbf{70.74}$ & $235.8^1$ & $440.41^1$ & $\textbf{386.28}$ & $493.3^1$ & $476.17^1$ & $\textbf{362.7}$ & $547.51^1$ & $542.86^1$ & $\textbf{375.82}$ & $611.52^1$\\
german & $664.42^1$ & $\textbf{567.29}$ & $763.08^1$ & $531.22^1$ & $\textbf{512.17}$ & $936.35^1$ & $480.25^1$ & $\textbf{411.74}$ & $937.4^1$ & $506.24^1$ & $\textbf{451.19}$ & $1017.17^1$\\
heart & $116.32^1$ & $\textbf{97.64}$ & $155.65^1$ & $111.45^1$ & $\textbf{97.22}$ & $155.54^1$ & $82.09^1$ & $\textbf{64.5}$ & $169.22^1$ & $91.11^1$ & $\textbf{83.71}$ & $194.97^1$\\
hill-valley & $778.84^1$ & $\textbf{578.5}$ & $578.91$ & $604.4^1$ & $\textbf{536.96}$ & $\textbf{536.96}$ & $789.11^1$ & $\textbf{684.93}$ & $685.41$ & $984.32^1$ & $\textbf{748.74}$ & $\textbf{748.74}$\\
house-votes & $171.27^1$ & $\textbf{118.62}$ & $206.69^1$ & $204.17^1$ & $\textbf{176.01}$ & $250.3^1$ & $197.44^1$ & $\textbf{122.75}$ & $254.13^1$ & $149.98^1$ & $\textbf{104.16}$ & $272.02^1$\\
kr-vs-kp & $2380.23^1$ & $\textbf{2081.5}$ & $2495.92^1$ & $1757.88^1$ & $\textbf{1624.11}$ & $2318.33^1$ & $1609.81^1$ & $\textbf{1147.79}$ & $2344.12^1$ & $1905.09^1$ & $\textbf{1363.94}$ & $2869.24^1$\\
mushroom & $5578.79^1$ & $\textbf{3506.46}$ & $6603.53^1$ & $3569.52^1$ & $\textbf{2546.44}$ & $4430.6^1$ & $5178.54^1$ & $\textbf{2520.05}$ & $5872.11^1$ & $4269.31^1$ & $\textbf{2040.03}$ & $6114.33^1$\\
sonar & $142.89^1$ & $\textbf{130.41}$ & $147.37^1$ & $72.75$ & $\textbf{68.58}$ & $97.75^1$ & $119.6^1$ & $\textbf{103.21}$ & $153.96^1$ & $84.83^1$ & $\textbf{75.64}$ & $103.28^1$\\
wdbc & $150.34$ & $\textbf{142.44}$ & $363.01^1$ & $116.01$ & $\textbf{103.1}$ & $320.44^1$ & $151.82^1$ & $\textbf{95.9}$ & $377.93^1$ & $109.95^1$ & $\textbf{68.94}$ & $360.53^1$\\
\bottomrule
\end{tabular}
\end{lrbox}
\scalebox{0.75}{\usebox{\rst}}
\label{testMMtc}
\end{table*}
In summary, although we implemented the three compared frameworks with a preliminary and less effective model, the result reported in the paper confirms the efficacy of Corollary \ref{corollary}. Once we are equipped with particular designed method can solve the minimax problem over only two cost matrices effectively, it would be very exciting to see the full advantage of the \textbf{SP} framework.

\section{Conclusions and Discussions}
For many real-world cost-sensitive learning problems, the costs associated with misclassifications are uncertain in nature. Many existing cost-sensitive learning algorithms, which require the exact cost information (e.g., a unique cost matrix) being available, are not applicable for these problems. In this paper, we consider the situation where the cost information is provided as a set of cost matrices, and aim to achieve the minimax classifier over the cost matrices. It is theoretically proved that the classifier with minimax total cost is either the optimal classifier for a single cost matrix in the set, or the minimax classifier over a pair of cost matrices in the set. This result immediately leads to a framework for achieving minimax classifier over arbitrary number of cost matrices. Furthermore, it is also applicable in case that the cost information is provided as an infinite set, e.g., intervals, by combining with an appropriate sampling/discretization procedure. Preliminary empirical study has justified the efficacy of the framework.

Although there exist a lot of algorithms for standard cost-sensitive learning problems, achieving minimax classifier over a pair of cost matrices remains the major technical obstacle. Therefore, novel algorithms should be developed for this purpose to exploit the usefulness of our framework to the full extent. Furthermore, the theoretical analysis conducted in this work needs to be extended to multi-class problems so that the resultant framework can be generalized. These issues will be investigated in the future.

\bibliography{MyCollection}
\bibliographystyle{IEEEtran}
\end{document}